\documentclass{article}
\usepackage{ijcai16}

\usepackage{times}
\usepackage{url}
\usepackage{algorithm}
\usepackage{algorithmic}
\usepackage{amsmath, amsthm, amssymb}
\usepackage{graphics}
\usepackage{epsfig}
\usepackage{subfigure}
\usepackage{float}
\usepackage{color}

\begin{document}
	
\hyphenpenalty = 10000
\tolerance= 2000

\title{Learning of Generalized Low-Rank Models: A Greedy Approach}
\author{Quanming Yao \quad\quad James T. Kwok \\
	Department of Computer Science and Engineering\\
	Hong Kong University of Science and Technology \\
	Hong Kong\\
	\{qyaoaa, jamesk\}@cse.ust.hk}

\pdfinfo{
	/Title (Learning of Generalized Low-Rank Models: A Greedy Approach)
	/Author (Quanming Yao, James T.Kwok) }

\maketitle

\newcommand{\fix}{\marginpar{FIX}}
\newcommand{\new}{\marginpar{NEW}}

\newcommand{\rank}[1]{\text{rank}(#1)}
\newcommand{\st}{\;\text{s.t.}\;}
\newcommand{\FN}[1]{\|#1\|_F}
\newcommand{\SO}[1]{P_{\Omega}(#1)}
\def\t{\theta}
\newcommand{\R}{\mathbb{R}}
\renewcommand{\vec}[1]{\text{vec}(#1)}
\newcommand{\Tr}[1]{\text{Tr}(#1) }
\newcommand{\SP}[1]{||#1||_1}
\newcommand{\SQ}[1]{\left\| #1 \right\|_2}
\newcommand{\NN}[1]{\| #1 \|_*}

\newtheorem{theorem}{Theorem}
\newtheorem{lemma}[theorem]{Lemma}
\newtheorem{proposition}[theorem]{Proposition}
\newtheorem{corollary}[theorem]{Corollary}
\newtheorem{remark}{Remark}
\newtheorem{definition}{Definition}

\maketitle

\begin{abstract}
	Learning of low-rank matrices is fundamental to many machine learning applications.  A
	state-of-the-art algorithm
	is the 
	rank-one matrix pursuit (R1MP). 
	However, it can only be used in matrix completion problems with  the square loss. In this
	paper, we develop a more flexible greedy algorithm for generalized
	low-rank models whose optimization objective can be smooth or nonsmooth,
	general convex or
	strongly convex.
	The proposed algorithm has low per-iteration time complexity and fast convergence rate.
	Experimental results 
	show that it 
	is much faster
	than the state-of-the-art,
	with comparable or even better prediction performance.
\end{abstract}


\section{Introduction}

In many machine learning problems, the data samples can be naturally represented as low-rank
matrices. 
For example, in  recommender systems, 
the ratings matrix is low-rank
as users (and items) tend to form groups.
The prediction of unknown ratings is then 
a low-rank matrix completion problem
\cite{candes2009exact}.
In social network analysis,
the network can be represented by a matrix with
entries representing similarities between 
node pairs.
Unknown links are treated as missing values and predicted as in matrix completion \cite{chiang2014prediction}.
Low-rank matrix learning also have applications in 
image and video processing \cite{candes2011robust},
multitask learning \cite{argyriou2006multi}, multilabel learning 
\cite{tai2012multilabel},
and robust matrix factorization \cite{eriksson2012efficient}.

The low-rank matrix optimization problem is NP-hard
\cite{recht2010guaranteed},
and direct minimization is difficult.  To
alleviate this problem, one common approach is to factorize the target $m\times n$ matrix $X$ as a product of two
low-rank matrices $U$ and $V$,
where $U \in \R^{m \times r}$ and $V \in \R^{n \times r}$ with $r \ll \min\{m,n\}$.
Gradient descent and alternating minimization are often used for
optimization \cite{srebro2004maximum,eriksson2012efficient,wen2012solving}.  
However, as the objective is not jointly convex in $U$ and $V$, this approach can suffer from slow convergence \cite{hsieh2014nuclear}.

Another approach is to replace the 
matrix rank 
by the nuclear norm (i.e., sum of  its
singular values).
It is known that 
the nuclear norm 
is the tightest convex envelope of  the matrix rank 
\cite{candes2009exact}.
The resulting optimization problem is convex, and 
popular convex optimization solvers
such as the proximal gradient algorithm \cite{beck2009fast} 
and Frank-Wolfe algorithm \cite{jaggi2013revisiting} can be used.
However,
though convergence properties can be guaranteed, singular value decomposition 
(SVD) is required in each iteration
to generate the next iterate.
This can be prohibitively expensive when the target matrix is large.
Moreover,
nuclear norm regularization often leads to biased estimation.
Compared to factorization approaches,
the obtained rank can be much higher and
the prediction performance 
is inferior 
\cite{mazumder2010spectral}. 

Recently, greedy algorithms have been explored for low-rank optimization
\cite{shalev2011large,wang2014rank}.
The idea is similar to orthogonal matching pursuit (OMP) \cite{pati1993orthogonal} in sparse coding.
For example, the state-of-the-art in matrix completion is
the rank-one matrix pursuit (R1MP) algorithm \cite{wang2014rank}.  
In each iteration, it
performs an efficient rank-one SVD on a sparse matrix, and
then greedily adds a rank-one matrix to the matrix estimate. 
Unlike
other algorithms
which typically require a lot of iterations,
it only 
takes $r$ iterations 
to obtain a rank-$r$ solution.
Its prediction performance is also comparable or even
better than others.

However, R1MP  is only designed
for matrix completion with the square loss. As recently discussed in 
\cite{udell2014generalized}, 
different loss functions may be required in different learning scenarios.
For example,
in link prediction,
the presence or absence of a link is naturally represented by a binary variable, and
the logistic loss is thus more appropriate.
In robust  matrix learning applications, 
the $\ell_1$ loss or Huber loss can be used to reduce sensitivity to outliers
\cite{candes2011robust}.
While computationally R1MP can be used with these loss functions, its convergence analysis is
tightly based
on OMP (and thus the square loss), and cannot be easily extended.

This motivates us to develop more general greedy algorithms that can be used in a
wider range of low-rank matrix learning scenarios. In particular,
we consider low-rank matrix optimization problems of the form
\begin{align}
\min_{X}  
f(X) \;:\; \rank{X} \le r,
\label{eq:mdl}
\end{align}
where $r$ is the target rank, and the objective
$f$ can be smooth or nonsmooth,
(general) convex or strongly convex.
The proposed algorithm is an extension of R1MP, and can be reduced to R1MP when $f$ is the square loss.
In general, when $f$ is convex and Lipschitz-smooth,
convergence is guaranteed
with a 
rate  of
$O(1/T)$. 
When $f$ is strongly convex, this is improved to 
a linear rate.
When $f$
is nonsmooth, we obtain
a $O(1/\sqrt{T})$ rate for (general) convex objectives and
$O(\log(T)/T)$ for strongly convex objectives.
Experiments on large-scale data sets 
demonstrate that the proposed algorithms are much faster than the state-of-the-art,  
while achieving comparable or even better prediction performance.

\noindent
\textbf{Notation}:
The transpose of vector / matrix is denoted by the superscript $\ ^{\top}$.
For matrix
$A=[A_{ij}] \in \R^{m \times n}$
(without loss of generality, we assume that $m \le n$),
its Frobenius norm is $\FN{A} = \sqrt{\sum_{i,j} A_{ij}^2}$, 
$\ell_1$ norm
is 
$\SP{X}= \sum_{i,j} |X_{ij}|$,
nuclear norm is $\NN{A} = \sum_i \sigma_i(A)$, where $\sigma_i(A)$'s are
the singular values, and
$\sigma_{\max}(A)$ is its largest singular value.
For two vectors $x,y$, the inner product $\langle x, y\rangle= \sum_{i} x_i y_i$;
whereas 
for two matrices $A$ and $B$, 
$\langle A, B\rangle = \sum_{i,j} A_{ij} B_{ij}$.
For 
a smooth function
$f$, $\nabla f$ denotes its gradient. When $f$ is
convex but nonsmooth, $g \in \{u \,|\, f(y) \ge f(x) + \langle u, x - y \rangle \}$ is
its subgradient at $x$.
Moreover,
given $\Omega \in \{0,1\}^{m \times n}$,
$[\SO{A}]_{ij} = A_{ij}$ if $\Omega_{ij} = 1$,
and 0 otherwise.


\section{Review: Rank-One Matrix Pursuit}

The rank-one matrix pursuit (R1MP) algorithm \cite{wang2014rank} is designed for
matrix completion 
\cite{candes2009exact}.
Given a partially observed $m \times n$ matrix $O =[O_{ij}]$, indices of the observed
entries are contained in the matrix $\Omega \in \{0,1\}^{m \times n}$, where $\Omega_{ij} =1$
if $O_{ij}$ is observed and 0 otherwise.
The goal is to find a low-rank matrix that is most similar to $O$  at the observed entries.
Mathematically, this is formulated as the following optimization problem:
\begin{equation} \label{eq:mc} 
\min_{X}  
\sum_{(i,j)\in\Omega} (X_{ij} - O_{ij})^2
\;:\;
\rank{X} \le r,
\end{equation}
where $r$ is the target rank.
Note that the square loss has to be used in R1MP.

The key observation is that if $X$ has rank $r$, it can be written as the sum of $r$ rank-one matrices, i.e.,
$X = \sum_{i = 1}^r \theta_i u_i v_i^{\top}$,
where $\theta_i\in\R$ and $\SQ{u_i} = \SQ{v_i} = 1$.
To solve (\ref{eq:mc}),
R1MP 
starts with an empty estimate.
At the $t$th iteration,
the $(u_t, v_t)$ pair  that is most correlated with the current residual
$R_t = \SO{O - X_{t-1}}$
is greedily
added. It can be easily shown that 
this $(u_t, v_t)$ pair  
are the leading left and right singular vectors of
$R_t$, and 
can be efficiently obtained from the rank-one SVD of
$R_t$.
After adding this new $u_t v_t^\top$ basis matrix, all coefficients 
of the current
basis 
can be updated as
\begin{equation} \label{eq:r1mprefine1}
\!\! (\theta_1,\dots,\theta_t) \leftarrow
\arg\min_{\theta_1,\dots,\theta_t} \;
\left\|P_\Omega\left(\sum_{i = 1}^t \theta_i u_i v_i^{\top} - O \right)\right\|_F^2\!\!\!.\!\!\!
\end{equation} 
Because of the use of the square loss, this
is a simple least-squares regression problem with closed-form solution. 

To save computation,
R1MP also has an economic variant. This only updates the 
combination coefficients of the current estimate and the rank-one update matrix as:
\begin{eqnarray}
\!\!\! (\mu,\rho) \!\leftarrow\! \arg\min_{\mu,\rho}  \;
\left\|P_\Omega\left(\mu \sum_{i = 1}^{t - 1} \theta_i u_i v_i^{\top} \!+\! \rho u_t v_t^{\top} \!- \!
O\right)\right\|_F^2\!\!\!.
\label{eq:r1mprefine2}
\end{eqnarray} 
The whole procedure is shown in Algorithm~\ref{alg:r1mp}.

\begin{algorithm}[ht]
	\caption{R1MP
		\protect\cite{wang2014rank}.}
	\begin{algorithmic}[1]
		\STATE{{\bf Initialize:} $X_0 = 0$;}
		\FOR{$t = 1, \dots, T$}
		\STATE{$R_t = \SO{O - X_{t - 1}}$};
		\STATE{$[u_t, s_t, v_t] = \text{rank1SVD}(R_t)$};
		\STATE{update coefficients}
		using \eqref{eq:r1mprefine1}
		(standard version), or
		\eqref{eq:r1mprefine2} (economic version); 
		\STATE $X_t = \sum_{i = 1}^t \theta_i u_i v_i^{\top}$;
		\ENDFOR
		\RETURN $X_T$.
	\end{algorithmic}
	\label{alg:r1mp}
\end{algorithm}

Note that each R1MP iteration is 
computationally
inexpensive. Moreover, as the matrix's rank is increased by one in
each iteration, 
only $r$ iterations are needed in order to obtain a rank-$r$ solution.
It can also be shown that the 
residual's norm 
decreases at a linear rate, i.e., $\FN{R_t}^2\leq
\gamma^{t-1}\FN{\SO{O}}^2$ for some $\gamma\in (0,1)$.


\section{Low-Rank Matrix Learning with Smooth Objectives}
\label{sec:smooth}

Though 
R1MP is
scalable, it
can only be used for matrix completion with the square loss.
In this Section, we extend R1MP to problems with more general, smooth 
objectives. Specifically, we only assume that the objective $f$ is convex and
$L$-Lipschitz smooth.
This will be further extended to nonsmooth objectives in Section~\ref{sec:nonsmooth}.

\begin{definition}
	$f$ is $L$-Lipschitz smooth 
	if $f(X) \le f(Y) + \left\langle X - Y, \nabla f(Y) \right\rangle + \frac{L}{2} \FN{X - Y}^2$ for any $X, Y$.
\end{definition}


\subsection{Proposed Algorithm}
\label{sec:algo}

Let the matrix iterate 
at the $t$th iteration
be
$X_{t - 1}$.
We follow the gradient direction $\nabla f(X_{t - 1})$ of the objective $f$, and find the rank-one matrix $M$
that is most correlated with $\nabla f(X_{t - 1})$:
\begin{equation} \label{eq:M}
\max_{M} \; \langle M, \nabla f(X_{t - 1}) \rangle \;:\; \rank{M} = 1, \FN{M} = 1. 
\end{equation} 

Similar to \cite{wang2014rank}, 
its optimal solution 
is given by 
$u_t v_t^{\top}$, where
$(u_t, v_t)$ 
are the leading left and right 
singular vectors 
of $\nabla f(X_{t - 1})$.
We then set the coefficient $\theta_t$ for
this new rank-one update matrix to $- s_t/L$, where $s_t$ is the 
singular value
corresponding  to
$(u_t, v_t)$.
Optionally, 
all the
coefficients 
$\theta_1,\dots,\theta_t$ 
can be refined
as
\begin{align}
(\theta_1,\dots,\theta_t) \leftarrow
\arg\min_{\theta_1,\dots,\theta_t} \;
f \left( \sum_{i = 1}^t \theta_i u_i v_i^{\top} \right).
\label{eq:refine1}
\end{align}

As in R1MP,
an economic variant is to update
the coefficients as
$[\mu\theta_1,\dots, \mu\theta_{t-1}, \rho]$, where $\mu$ and $\rho$ are obtained as
\begin{eqnarray}
(\mu, \rho) \leftarrow
\arg\min_{\mu, \rho} \;
f\left( \mu \sum_{i = 1}^{t - 1} \theta_i u_i v_i^{\top}
+ \rho \, u_t v_t^{\top} \right).
\label{eq:refine2}
\end{eqnarray}

The whole procedure, which will be called
``greedy low-rank learning'' (GLRL),
is shown in Algorithm~\ref{alg:smooth}. Its economic variant will be called EGLRL. 
Obviously, on matrix completion problems with the square loss, Algorithm~\ref{alg:smooth} reduces to R1MP.

\begin{algorithm}[H]
	\caption{GLRL for low-rank matrix learning with smooth convex objective $f$.}
	\begin{algorithmic}[1]
		\STATE{{\bf Initialize:} $X_0 = 0$;}
		\FOR{$t = 1, \dots, T$}
		\STATE{$[u_t, s_t, v_t] = \text{rank1SVD}(\nabla f(X_{t-1}))$;}
		\STATE{$X_{t} = X_{t - 1} - \frac{s_t}{L} u_t v_t^{\top}$;}
		\STATE{(optional:) refine coefficients
			using \eqref{eq:refine1}
			(standard version) or
			\eqref{eq:refine2} (economic version);}
		\ENDFOR
		\RETURN $X_T$.
	\end{algorithmic}
	\label{alg:smooth}
\end{algorithm}

Note that \eqref{eq:refine1}, 
\eqref{eq:refine2} are smooth minimization problems 
(with $\leq r$ and 2 variables, respectively). 
As the target matrix is low-rank, 
$r$ should be small and thus
\eqref{eq:refine1},
\eqref{eq:refine2} can be
solved
inexpensively.
In the experiments, we use the 
popular
limited-memory BGFS (L-BFGS) solver \cite{nocedal2006numerical}.
Empirically,
fewer than five
L-BFGS
iterations are needed.
Preliminary experiments show that using more iterations does not improve performance.

Unlike R1MP, note that 
the coefficient refinement at step~5 is optional. 
Convergence results in Theorems~\ref{the:rate:strcvx} and \ref{the:rate:smtwc} 
below
still hold
even when 
this step is not performed.
However, as will be illustrated in Section~\ref{sec:link},
coefficient refinement 
is always beneficial 
in practice.
It results in a larger reduction of the objective in each iteration, and thus a better
rank-$r$
model after running for $r$ iterations.


\subsection{Convergence}
\label{sec:conv}

The analysis of R1MP is based on orthogonal matching pursuit \cite{pati1993orthogonal}.
This requires the condition $\frac{1}{2}\FN{\nabla f(X)}^2 = f(X)$, 
which only holds when $f$ is the square loss.
In contrast, our 
analysis 
for Algorithm~\ref{alg:smooth}
here 
is novel 
and
can be used for any 
Lipschitz-smooth
$f$.

The following Proposition shows that the objective
is decreasing in each iteration. Because of the lack of space, all the proofs will be omitted.

\begin{proposition} \label{pr:descent}
	If $f$ is $L$-Lipschitz smooth,
	\begin{align*}
	f(X_t) \le f(X_{t-1}) 
	- \frac{\gamma_{t - 1}^2}{2 L} \FN{\nabla f(X_{t-1})}^2, 
	\end{align*}
	where
	\begin{align}
	\gamma_{t - 1} = \frac{\sigma_{\max}(\nabla f(X_{t - 1}))}{\FN{\nabla f(X_{t - 1})}} \in
	\left[\frac{1}{\sqrt{m}}, 1\right].
	\label{eq:gamma}
	\end{align}
\end{proposition}

If $f$ is strongly convex,
a linear convergence rate can be obtained.

\begin{definition}
	$f$ is $\mu$-strongly convex 
	if $f(X) \ge f(Y) + \left\langle X - Y, \nabla f(Y) \right\rangle + \frac{\mu}{2} \FN{X - Y}^2$
	for any $X, Y$.
	\label{def:strcvx}
\end{definition}

\begin{theorem}
	\label{the:rate:strcvx}
	Let $X_*$ be the optimal solution of \eqref{eq:mdl}.
	If $f$ is $\mu$-strongly convex and $L$-Lipschitz smooth, 
	then
	\begin{align*}
	f(X_T) - f(X_*) \le \left( 1 - \frac{d_1^2 \mu }{L} \right)^T
	\left[ f(X_0) - f(X_*) \right],
	\end{align*}
	where $d_1 = \min_{t = 1}^T \gamma_t$.
\end{theorem}

If $f$ is only (general) convex,
the following shows that Algorithm~\ref{alg:smooth} converges at a slower 
$O(1/T)$
rate.

\begin{theorem} \label{the:rate:smtwc}
	If $f$ is (general) convex and $L$-Lipschitz smooth, then
	\begin{align*}
	f(X_T) - f(X_*)
	\le \frac{2 d_2^2 L \left[f(X_0) - f(X_*)\right]}
	{d_1^2 T \left[f(X_0) - f(X_*)\right] + 2 d_2^2 L},
	\end{align*}
	where $d_2 = \max_{t = 1}^T \FN{X_t - X_*}$.
\end{theorem}

The square loss 
in (\ref{eq:mc})
is only general convex and $1$-Lipschitz smooth.
From Theorem~\ref{the:rate:smtwc}, one would expect GLRL to only have a sublinear convergence rate of
$O(1/T)$ on matrix completion problems.
However, our analysis can be refined in this special case. The following Theorem shows that
a linear rate can indeed be obtained, which also agrees with  
Theorem~3.1 of
\cite{wang2014rank}.

\begin{theorem}
	\label{the:rate:matcomp}
	When $f$ is the square loss,
	$f(X_T) - f(X_*) \le (1 - d_1^2 )^T \left[f(X_0) - f(X_*)\right]$.
\end{theorem}


\subsection{Per-Iteration Time Complexity}
\label{sec:time-smooth}

The per-iteration time complexity of Algorithm~\ref{alg:smooth} is low. 
Here, we take the link prediction problem in
Section~\ref{sec:link} as an example. 
With $f$ defined only on the observed entries of the link matrix,
$\nabla f$ is sparse, 
and computation of $\nabla f(X_{t - 1})$ in step~3 takes $O(\SP{\Omega})$ time.
The rank-one SVD on $\nabla f(X_{t - 1})$ can be obtained by the power method \cite{halko2011finding}
in $O(\SP{\Omega})$ time.
Refining coefficients using \eqref{eq:refine1} 
takes $O(t\SP{\Omega})$ time for the $t$th iteration.
Thus,
the total per-iteration time complexity of GLRL is $O(t\SP{\Omega})$.
Similarly, 
the per-iteration time complexity of EGLRL is 
$O(\SP{\Omega})$.
In comparison,
the state-of-the-art AIS-Impute algorithm \cite{quan2015impute} (with a convergence rate of
$O(1/T^2)$) takes
$O(r^2\SP{\Omega})$ time in each iteration, whereas the alternating minimization approach in 
\cite{chiang2014prediction} (whose
convergence rate is unknown)
takes $O(r\SP{\Omega})$   time per iteration.


\subsection{Discussion}

To learn the generalized low-rank model,
\citeauthor{udell2014generalized} (\citeyear{udell2014generalized}) followed the
common approach of factorizing the
target matrix as a product of two low-rank matrices and then performing
alternating minimization 
\cite{srebro2004maximum,eriksson2012efficient,wen2012solving,chiang2014prediction,yu2014large}.
However, this may not
be very efficient, and
is much slower than R1MP 
on matrix completion problems
\cite{wang2014rank}.
More
empirical comparisons 
will be demonstrated 
in Section~\ref{sec:link}.

Similar to R1MP,
the greedy efficient component optimization (GECO) \cite{shalev2011large}
is also based on greedy approximation 
but can be used with any smooth objective. 
However, GECO is even slower than R1MP \cite{wang2014rank}.
Moreover, it does not have convergence guarantee.


\section{Low-Rank Matrix Learning with Nonsmooth Objectives}
\label{sec:nonsmooth}

Depending on the application,
different (convex) nonsmooth loss functions
may be used in generalized low-rank matrix models
\cite{udell2014generalized}.
For example, 
the $\ell_1$ loss is useful
in robust matrix factorization \cite{candes2011robust},
the scalene loss in quantile regression \cite{koenker2005quantile}, 
and the hinge loss in multilabel learning \cite{yu2014large}.
In this Section, we extend the GLRL algorithm, with simple modifications, to nonsmooth objectives.


\subsection{Proposed Algorithm}

As the objective is nonsmooth, one has to use 
the subgradient $g_t$ of $f(X_{t-1})$ 
at the $t$th iteration
instead  of the gradient in Section~\ref{sec:smooth}.
Moreover, refining the coefficients as in  (\ref{eq:refine1}) or (\ref{eq:refine2}) will now involve nonsmooth optimization, which is much harder. Hence,  we do not optimize the coefficients.
To ensure  convergence, a sufficient reduction in the objective in each iteration is still
required.  To achieve this, instead of just adding a rank-one matrix, we add a rank-$k$
matrix (where $k$ may be greater than 1).
This matrix should be most similar to $g_t$, which can be easily obtained as:
\begin{equation} \label{eq:M2}
M^* \equiv
\arg\min_{M : \rank{M} = k} \FN{M - g_t}^2 =
\sum_{i = 1}^k s_i u_i v_i^{\top},  
\end{equation} 
where $\{(u_1, v_1),\dots,(u_k, v_k)\}$ are the 
$k$ leading left and right singular vectors of $g_t$,
and $\{s_1, \dots, s_k \}$ are corresponding singular values.
The proposed procedure is shown in Algorithm~\ref{alg:nonsmooth}. The stepsize in step~3 is given
by 
\begin{equation} 
\label{eq:rate} 
\eta_t = \left\{ \begin{array}{ll}
c_1/t& 
\text{if $f$ is $\mu$-strongly convex} \\
c_2/\sqrt{t} & 
\text{if $f$ is (general) convex} 
\end{array} \right.,
\end{equation} 
where $c_1 \ge 1/\mu$ and $c_2 > 0$.

\begin{algorithm}[ht]
	\caption{GLRL for low-rank matrix learning with nonsmooth objective $f$.}
	\begin{algorithmic}[1]
		\STATE{{\bf Initialize:} $X_0 = 0$ and choose $\nu \in \left(0, 1\right)$;}
		\FOR{$t = 1, \dots, T$}
		\STATE{set $\eta_t$ as in \eqref{eq:rate};}
		\STATE{compute subgradient $g_t$ of $f(X_{t-1})$, $h_t = 0$;}
		\FOR{$i = 1, 2,\dots$}
		\STATE{$[u_i, s_i, v_i] = \text{rank1SVD}(g_t - h_t)$;}
		\STATE{$h_t = h_t + s_i u_i v_i^{\top}$;}
		\IF{$\FN{g_t - h_t}^2 \le \nu \FN{g_{t - 1} - h_{t-1}}^2 $}
		\STATE{break;}
		\ENDIF
		\ENDFOR
		\STATE {$X_{t} = X_{t - 1} - \eta_t h_t$;}
		\ENDFOR
		\RETURN $X_T$.
	\end{algorithmic}
	\label{alg:nonsmooth}
\end{algorithm}


\subsection{Convergence}

The following Theorem shows that
when $f$ is nonsmooth and strongly convex,
Algorithm~\ref{alg:nonsmooth} has
a convergence rate  of $O(\log{T}/T)$.

\begin{theorem} \label{the:rate:gs}
	Assume that $f$ is $\mu$-strongly convex,
	and 
	$\FN{g_t} \le b_1$ for some $b_1$
	($t = 1, \dots, T$), 
	then
	\[ \min_{t = 0, \dots, T} f(X_t) - f(X_*)
	\le 
	\frac{1}{T}
	\left( \left( 1 + \log T \right) \frac{c_1 b_1^2}{2} + b_2 \right), \]
	where 
	$c_1$ is as defined in \eqref{eq:rate}, and 
	$b_2$ is a constant (depending on $X_0$, $\mu$, $\nu$ and $c_1$).
\end{theorem}

When $f$ is only (general) convex, 
the following Theorem shows that
the rate 
is reduced to $O(1/\sqrt{T})$.

\begin{theorem} \label{the:rate:gc}
	Assume  that
	$f$ is (general) convex, 
	$\FN{g_t} \le b_1$ for some $b_1$ and $\FN{X_t - X_*} \le b_3$ for
	some $b_3$
	($t = 1, \dots, T$), 
	then
	\begin{align*}
	\min_{t = 0,\dots,T} f(X_t) - f(X_*)
	\le \frac{c_2(b_1^2 + b_3^2)}{2 \sqrt{T}} - \frac{2 b_1 b_3}{(1 - \sqrt{\nu}) T},
	\end{align*}
	where $c_2$ is as defined in \eqref{eq:rate}.
\end{theorem}

For other convex nonsmooth optimization problems, the same 
$O(\log(T)/T)$ rate for strongly convex objectives and
and $O(1/\sqrt{T})$ rate for general convex objectives have also been observed
\cite{shalev2011pegasos}.
However, their analysis is for different problems, and cannot be readily applied to our low-rank matrix learning problem here.

\begin{figure*}[ht]
	\centering
	\subfigure[Epinions (training).] 
	{\includegraphics[width = 0.22\textwidth]{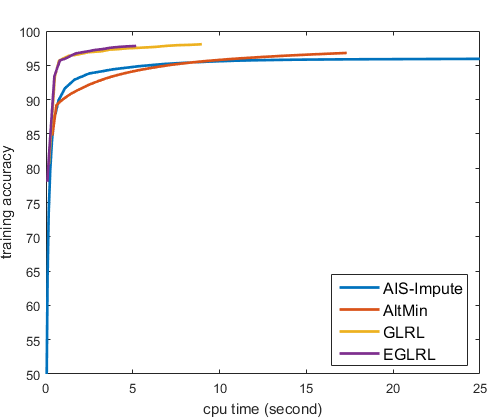}}
	\quad	
	\subfigure[Epinions (testing).] 
	{\includegraphics[width = 0.22\textwidth]{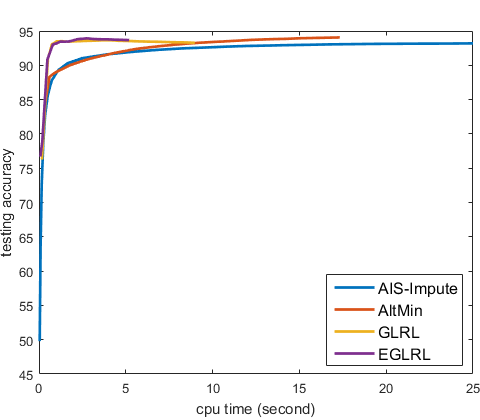}}
	\quad
	\subfigure[Slashdot (training).] 
	{\includegraphics[width = 0.22\textwidth]{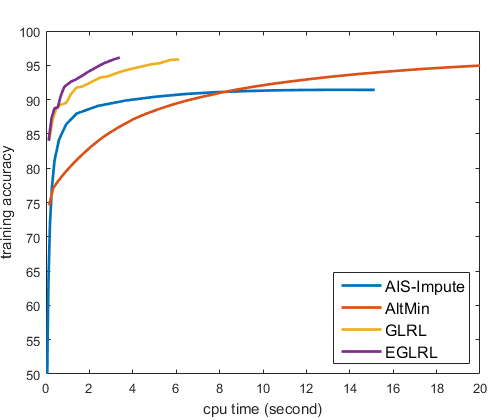}}
	\quad
	\subfigure[Slashdot (testing).] 
	{\includegraphics[width = 0.22\textwidth]{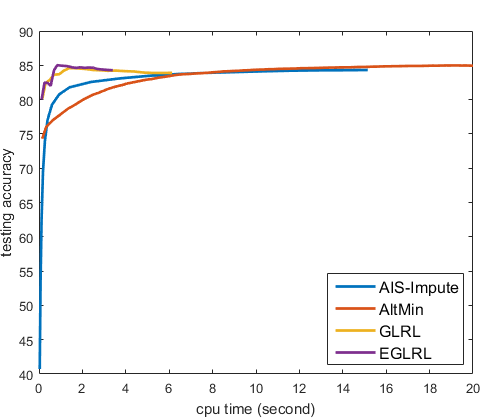}}
	\vspace{-.1in}
	\caption{Training and testing accuracies (\%) vs CPU time (seconds) on the Epinions (left) and Slashdot (right) data sets.}
	\label{fig:link:accu}
\end{figure*}


\subsection{Per-Iteration Time Complexity}
\label{sec:time-nonsmooth}

To study the per-iteration time complexity, we 
take the robust matrix factorization problem in Section~\ref{sec:robust} as an example.
The main computations are on steps~4 and 6.
In step~4,
since the subgradient $g_t$ is sparse (nonzero only at the observed entries),
computing $g_t$
takes $O(\SP{\Omega})$ time.
At outer iteration $t$ and inner iteration $i$, 
$g_t$ 
in step~6
is sparse and $h_t$ has low rank (equal to $i - 1$). Thus,
$g_t - h_t$ admits the so-called ``sparse plus low-rank'' structure \cite{mazumder2010spectral,qyao2015icdm}. This 
allows matrix-vector multiplications and 
subsequently rank-one SVD to be performed much more efficiently.
Specifically,
for any $v \in \R^n$,
the multiplication
$(g_t - h_t)v$ 
takes only $O(\SP{\Omega} + (i - 1)n)$ time
(and similarly for the multiplication $u^{\top}(g_t - h_t)$ with any $u \in \R^m$),
and rank-one SVD
using the power method 
takes $O(\SP{\Omega} + (i - 1)n)$ time.
Assuming that $i_t$ inner iterations are run at (outer)
iteration $t$,
it takes a total of $O(i_t \SP{\Omega} + (i_t - 1)^2n)$ time. 
Typically,
$i_t$ is small
(empirically, usually 1 or 2).

In comparison,
though 
the ADMM algorithm in \cite{lin2010augmented}
has a faster 
$O(1/T)$
convergence rate,
it needs SVD and takes
$O(m^2 n)$ time in each iteration.
As for the Wiberg algorithm 
\cite{eriksson2012efficient},
its convergence rate
is unknown and
a linear program with $m r + \SP{\Omega}$ variables
needs to be solved
in each iteration.
As will be seen in Section~\ref{sec:robust}, this is much slower than GLRL.


\section{Experiments}
\label{sec:expt}

In this section, 
we compare the proposed algorithms with the state-of-the-art on 
link prediction 
and robust matrix factorization.
Experiments are performed on a PC with Intel i7 CPU and 32GB RAM.  
All the codes are in Matlab.


\subsection{Social Network Analysis}
\label{sec:link}

Given a graph with $m$ nodes
and an incomplete
adjacency  matrix $O \in \{\pm 1\}^{m \times m}$,
link prediction aims to recover a low-rank matrix $X \in \R^{m \times m}$ 
such that
the signs of 
$X_{ij}$'s 
and $O_{ij}$'s
agree on most of the observed entries.
This can be formulated as the following optimization problem 
\cite{chiang2014prediction}:
\begin{equation} \label{eq:logmatcomp}
\min_X \sum_{(i,j) \in \Omega} \log (1 + \exp(- X_{ij} O_{ij})) \;:\; \rank{X} \le r,
\end{equation} 
where $\Omega$ contains indices of the observed entries. Note that (\ref{eq:logmatcomp})
uses
the logistic loss, which is more appropriate as 
$O_{ij}$'s are binary.

The objective 
in (\ref{eq:logmatcomp})
is convex and smooth.
Hence, we compare
the proposed GLRL (Algorithm~\ref{alg:smooth} with coefficient update step (\ref{eq:refine1})) and
its economic variant EGLRL (using 
coefficient update step (\ref{eq:refine2})) 
with the following:
\begin{enumerate}
	\item AIS-Impute 
	\cite{quan2015impute}:
	This is an accelerated proximal gradient algorithm with further speedup based on 
	approximate SVD and the special ``sparse plus low-rank'' matrix structure in matrix
	completion;
	
	\item Alternating minimization
	(``AltMin'') \cite{chiang2014prediction}: This
	factorizes $X$ as a product of two low-rank matrices, 
	and then uses alternating gradient descent for optimization.
\end{enumerate}
As a further baseline, we also compare with the
GLRL variant that does not perform coefficient update.
We do not compare with greedy efficient component optimization (GECO) \cite{shalev2011large},
matrix norm boosting \cite{zhang2012accelerated} 
and active subspace selection \cite{hsieh2014nuclear},
as they have been shown to be slower 
than AIS-Impute and
AltMin
\cite{quan2015impute,wang2014rank}.

Experiments are performed on the
Epinions 
and 
Slashdot 
data
sets\footnote{\url{https://snap.stanford.edu/data/}} \cite{chiang2014prediction}
(Table~\ref{tab:logistic}).
Each row/column of the matrix $O$ corresponds to a user
(users with fewer than two observations are removed).
For Epinions, $O_{ij} = 1$ if user $i$ trusts user $j$, and $-1$ otherwise.
Similarly 
for Slashdot, $O_{ij} = 1$ if user $i$ tags user $j$ as friend, and $-1$ otherwise.

\begin{table}[ht]
	\centering
	\caption{Signed network data sets used. }
	\renewcommand{\arraystretch}{1.2}
	\begin{tabular}{c | c | c | c}
		\hline
		& \#rows & \#columns & \#observations    \\ \hline
		Epinions & 42,470 & 40,700   & $7.5 \times 10^5$ \\ \hline
		Slashdot & 30,670 & 39,196   & $5.0 \times 10^5$ \\ \hline
	\end{tabular}
	\label{tab:logistic}
\end{table}

As in \cite{wang2014rank}, we fix the number of power method iterations to $30$.
Following \cite{chiang2014prediction},
we 
use 10-fold cross-validation and
fix the rank $r$ to 40.
Note that AIS-Impute uses the nuclear norm regularizer and does not explicitly constrain the rank.
We select its regularization parameter so that its output rank is 40.
To obtain a rank-$r$ solution,
GLRL is simply run for
$r$ iterations.
For AIS-Impute and AltMin, 
they are stopped when the relative change in the objective is smaller than $10^{-4}$.
The output predictions are binarized by thresholding at zero.
As in \cite{chiang2014prediction}, 
the sign prediction accuracy 
is used as performance measure.

Table~\ref{tab:tencent} shows the sign prediction accuracy on the test set.
All methods,
except the GLRL variant that does not perform coefficient update,
have comparable prediction performance. 
However,  
as shown in Figure~\ref{fig:link:accu}, 
AltMin and AIS-Impute are much slower (as discussed in 
Section~\ref{sec:time-smooth}).
EGLRL 
has the lowest
per-iteration cost, and is 
also faster than GLRL.

\begin{table}[ht]
	\centering
	\caption{Testing sign prediction accuracy (\%) on link prediction.
		The best 
		and comparable 
		results (according to the pairwise t-test with 95\% confidence) are highlighted.}
	\label{tab:tencent}
	\renewcommand{\arraystretch}{1.2}
	\begin{tabular}{c| c | c}
		\hline
		& Epinions              & Slashdot              \\ \hline
		AIS-Impute     & 93.3$\pm$0.1          & 84.2$\pm$0.1          \\ \hline
		AltMin       & 93.5$\pm$0.1          & \textbf{84.9$\pm$0.1} \\ \hline
		GLRL w/o coef upd & 92.4$\pm$0.1          & 82.6$\pm$0.3          \\ \hline
		GLRL        & \textbf{93.6$\pm$0.1} & 84.1$\pm$0.4          \\ \hline
		EGLRL       & \textbf{93.6$\pm$0.1} & 84.4$\pm$0.3          \\ \hline
	\end{tabular}
\end{table}

\begin{figure*}[ht]
	\centering
	{\includegraphics[width = 0.22 \textwidth]{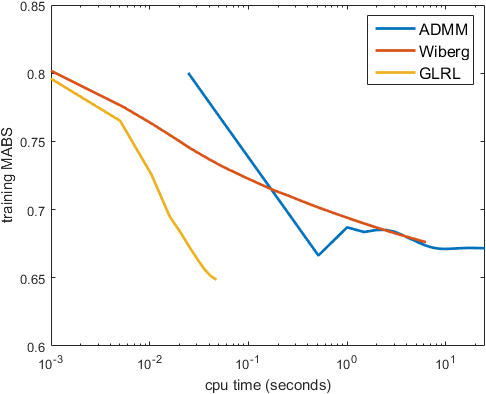}}
	\quad\quad\quad\quad\quad
	{\includegraphics[width = 0.22 \textwidth]{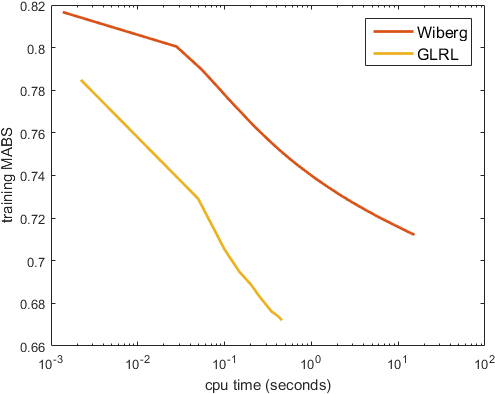}}
	\quad\quad\quad\quad\quad
	{\includegraphics[width = 0.22 \textwidth]{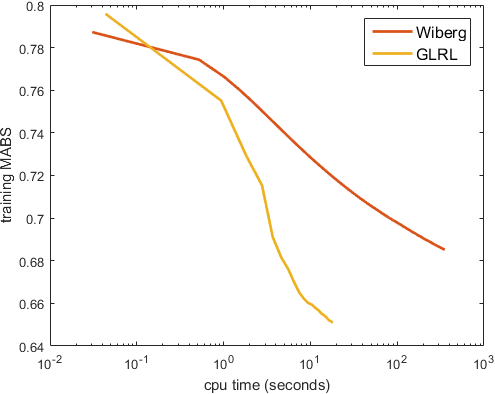}}
	\\
	\subfigure[100K.]{\includegraphics[width = 0.22 \textwidth]{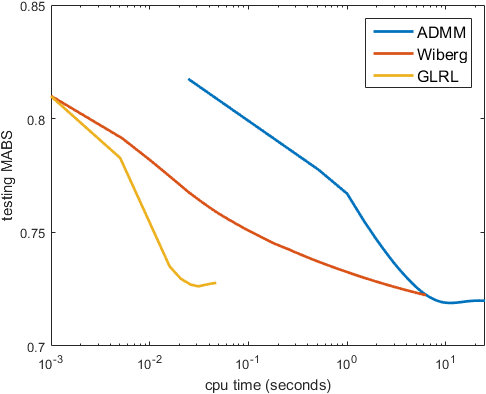}}
	\quad\quad\quad\quad\quad
	\subfigure[1M.]{\includegraphics[width = 0.22 \textwidth]{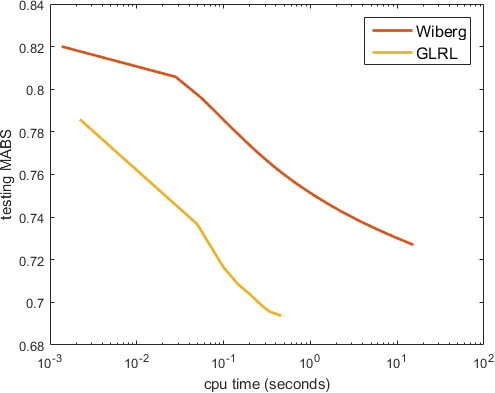}}
	\quad\quad\quad\quad\quad
	\subfigure[10M.]{\includegraphics[width = 0.22 \textwidth]{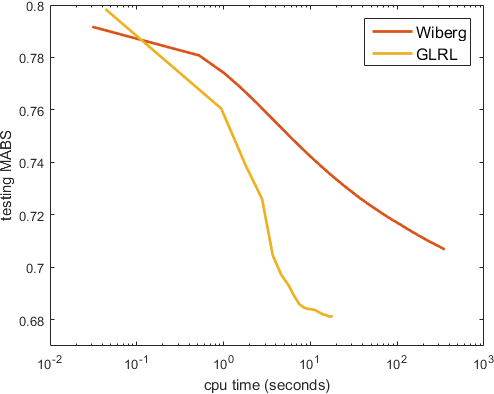}}
	\vspace{-.1in}
	\caption{Training (top) and testing (bottom) MABS vs CPU time (seconds) on the MovieLens data sets.}
	\label{fig:mlabs}
\end{figure*}


\subsection{Robust Matrix Factorization}
\label{sec:robust}

Instead of using 
the square loss,
robust matrix factorization  uses the $\ell_1$ loss to reduce
sensitivities to outliers
\cite{candes2011robust}.
This can be formulated as the optimization problem
\cite{lin2010augmented}:
\[ \min_{X} \sum_{(i,j)\in\Omega} \left| X_{ij} - O_{ij} \right| 
\;\text{s.t.}\; \rank{X} \le r. \]
Note that the objective is only general convex, and
its subgradient is bounded
($\le \FN{\Omega}$).
Since there is no smooth component in the objective, 
AIS-Impute and AltMin cannot be used.
Instead, 
we compare 
GLRL 
in Algorithm~\ref{alg:nonsmooth}
(with $\nu = 0.99$ and $c_2 =0.05$)
with 
the following:
\begin{enumerate}	
	\item Alternating direction method of multipliers (ADMM)\footnote{\url{http://perception.csl.illinois.edu/matrix-rank/Files/inexact_alm_rpca.zip}} \cite{lin2010augmented}:
	The rank constraint is replaced by the nuclear norm regularizer, and 
	ADMM \cite{boyd2011distributed} 
	is then used
	to solve the equivalent problem:
	$\min_{X,Y} \sum_{(i,j)\in\Omega} |X_{ij} - O_{ij}| + \lambda \NN{Y} \;\text{s.t.}\; Y = X$.
	
	\item The Wiberg algorithm \cite{eriksson2012efficient}:
	It factorizes $X$ into $U V^{\top}$ and optimizes them by linear
	programming. Here, we use 
	the linear programming solver
	in Matlab.
	
\end{enumerate}

Experiments are performed on the MovieLens data sets\footnote{\url{http://grouplens.org/data sets/movielens/}}
(Table~\ref{tab:sparse}),
which have been
commonly used for evaluating recommender systems \cite{wang2014rank}.
They contain ratings $\{1,2,\dots,5\}$ assigned by various users on
movies.
The setup is the same as in \cite{wang2014rank}.
$50\%$ of the ratings 
are randomly sampled
for training while the rest for testing.
The ranks used for the 100K, 1M, 10M data sets are
10, 10, and 20, respectively.
For performance evaluation,
we use the mean absolute error 
on the unobserved entries $\Omega^\perp$:
\[ \text{MABS} = \frac{1}{\SP{\Omega^\perp}} \sum_{(i,j) \in \Omega^\perp} |\hat{X}_{ij} -
O_{ij}|, \]
where $\hat{X}$ is the predicted matrix
\cite{eriksson2012efficient}.
Experiments are repeated five times with random training/testing splits.

\begin{table}[ht]
	\centering
	\caption{MovieLens data sets used in the experiments.}
	\vspace{5px}
	\renewcommand{\arraystretch}{1.2}
	\begin{tabular}{c | c | c | c}
		\hline
		     & \#users & \#movies & \#ratings \\ \hline
		100K & 943     & 1,682    & $10^5$    \\ \hline
		 1M  & 6,040   & 3,449    & $10^6$    \\ \hline
		10M  & 69,878  & 10,677   & $10^7$    \\ \hline
	\end{tabular}
	\label{tab:sparse}
\end{table}

Results are shown in Table~\ref{tab:matrobust}. As can be seen,
ADMM performs slightly better on the 100K data set,
and GLRL is more accurate than Wiberg.
However, ADMM is computationally
expensive as
SVD is required in each iteration. Thus, it cannot be run on the larger 1M and 10M data
sets.
Figure~\ref{fig:mlabs} shows the convergence of MABS with CPU time.
As can be seen,
GLRL is the fastest, which is then followed by
Wiberg, and ADMM is the slowest.

\begin{table}[ht]
	\centering
	\caption{Testing MABS on the MovieLens data sets. 
		The best results (according to the pairwise t-test with 95\% confidence) are highlighted.
		ADMM cannot converge in $1000$ seconds on the 1M and 10M data sets, and thus is not shown.}
	\label{tab:matrobust}
	\vspace{5px}
	\renewcommand{\arraystretch}{1.2}
	\begin{tabular}{c| c | c | c }
		\hline
		& 100K                     & 1M                       & 10M                      \\ \hline
		ADMM  & \textbf{0.717$\pm$0.004} & ---                      & ---                      \\ \hline
		Wiberg & 0.726$\pm$0.001          & 0.728$\pm$0.006          & 0.715$\pm$0.005          \\ \hline
		GLRL  & 0.724$\pm$0.004          & \textbf{0.694$\pm$0.001} & \textbf{0.683$\pm$0.001} \\ \hline
	\end{tabular}
\end{table}


\section{Conclusion}

In this paper, we propose an efficient greedy algorithm for the learning of generalized low-rank models.
Our algorithm is based on 
the state-of-art R1MP  algorithm, but 
allows the optimization objective to be smooth or nonsmooth, general convex or strongly  convex.
Convergence analysis shows that the proposed algorithm has fast convergence rates, and is compatible with those obtained on
other (convex) smooth/nonsmooth optimization problems.  Specifically, on smooth problems, it
converges with a rate of $O(1/T)$ on general convex problems and 
a linear rate on strongly convex problems. On nonsmooth problems, 
it converges with a rate of $O(1/\sqrt{T})$ on general convex problems and 
$O(\log(T)/T)$ rate on strongly convex problems. 
Experimental results on link prediction and robust matrix factorization 
show that the proposed algorithm achieves comparable or better prediction performance as
the state-of-the-art,
but is much faster.

\section*{Acknowledgments}

Thanks for helpful discussion from Lu Hou.
This research was supported in part by
the Research Grants Council of the Hong Kong Special Administrative Region
(Grant 614513).

{
\bibliographystyle{named}
\bibliography{paper}
}

\appendix

\section*{Proposition~\ref{pr:descent}}
\begin{proof}
At the $t$th iteration, we have $X_{t - 1} = \sum_{i = 1}^{t - 1}\theta_i u_i v_i^{\top}$.
Construct $\bar{\theta} = \left[ \bar{\theta}_i \right] \in \R^t$ as  
\begin{align*}
\bar{\theta}_i =
\begin{cases}
\theta_i & i < t \\
- s_t / L & i = t
\end{cases},
\end{align*}
and let $\bar{X}_{t} = \sum_{i = 1}^t \bar{\theta}_i u_i v_i^{\top}$.
Thus,
$\bar{X}_t - X_{t - 1} = - s_t u_t v_t^{\top}/L $.
As $f$ is $L$-Lipschitz smooth, 
\begin{eqnarray}
f(\bar{X}_t) & \le & f(X_{t-1}) + \langle \bar{X}_t - X_{t-1}, \nabla f(X_{t-1}) \rangle
\nonumber\\
&& + \frac{L}{2} \FN{\bar{X}_t - X_{t-1}}^2.  \label{eq:temp1} 
\end{eqnarray}
Note that $u_t^{\top} \nabla f(X_{t - 1}) v_t = s_t$. Together
with \eqref{eq:temp1},  we have
\begin{align}
f(\bar{X}_{t}) \le f(X_{t - 1}) - 
\frac{1}{2 L} \langle \nabla f(X_{t - 1}), u_t v_t^{\top} \rangle^2. 
\label{eq:temp2}
\end{align}

For the second term on the RHS of (\ref{eq:temp2}), 
\begin{align*}
\left\langle \nabla f(X_{t - 1}), u_t v_t^{\top} \right\rangle 
& = \gamma_{t-1} \FN{\nabla f(X_{t - 1})}\FN{u_t v_t^{\top}} \\
& = \gamma_{t-1} \FN{\nabla f(X_{t - 1})},
\end{align*}
where
$\gamma_{t-1}
= \frac{\left\langle \nabla f(X_{t-1}), u_t v_t^{\top} \right\rangle }{\FN{\nabla f(X_{t-1})}}$.
Note that $s_t$ is the largest singular value of $\nabla f(X_{t - 1})$.
Thus, $s_t \ge \FN{\nabla f(X_{t-1})}/\sqrt{m}$, 
and
\begin{align*}
\gamma_{t-1} = \frac{s_t}{\FN{\nabla f(X_{t - 1})}} \in [1/\sqrt{m}, 1]
\end{align*}

Then (\ref{eq:temp2}) becomes
\begin{align*}
f(\bar{X}_t)
\le f(X_{t - 1})
- \frac{\gamma_{t-1}^2}{2 L} \FN{\nabla f(X_{t - 1})}^2.
\end{align*}

In Algorithm~\ref{alg:smooth}, 
$X_t$ is obtained by minimizing $\theta_t$ over \eqref{eq:refine1} or \eqref{eq:refine2}.
Once we warm start $\theta$ using $\bar{\theta}$, 
we can ensure that
$f(X_t) \le f(\bar{X}_t)$, 
and thus the Proposition holds.
\end{proof}


\section*{Theorem~\ref{the:rate:strcvx}}

\begin{lemma} \label{lem:app1}
If $f$ is $\mu$-strongly convex, $f(Y) \ge f(X) - \frac{1}{2 \mu} \FN{\nabla f(X)}^2$
for any $X, Y$.
\end{lemma}

\begin{proof}
Since $f$ is $\mu$-strongly convex, 
\begin{align*}
f(Y) 
& \ge f(X) + \left\langle \nabla f(X), Y - X \right\rangle 
+ \frac{\mu}{2} \FN{Y - X}^2 \\
& \ge \min_{\bar{Y}}
f(X) + \left\langle \nabla f(X), \bar{Y} - X \right\rangle 
+ \frac{\mu}{2} \FN{\bar{Y} - X}^2.
\end{align*}

The minimum is achieved at $\bar{Y} = X - \frac{1}{\mu}\nabla f(X)$,
and $f(\bar{Y}) =f(X) - \frac{1}{2 \mu} \FN{\nabla f(X)}^2$.
\end{proof}

\begin{proof}
	On optimal $\nabla f(X_*) = 0$, using Lemma~\ref{lem:app1}:
	\begin{eqnarray}
	\FN{ \nabla f(X_{t-1}) - \nabla f(X_*) }^2 \ge 2 \mu \left[ f(X_{t-1}) - f(X_*) \right].
	\label{eq:temp3}
	\end{eqnarray}
	Combine it with Proposition~\ref{pr:descent}, then
	\begin{align*}
	& f(X_t) - f(X_*) \\
	& \le \left[ f(X_{t - 1}) - f(X_*) \right] - \frac{(\gamma_{t-1})^2}{2 L} \FN{\nabla f(X_{t - 1})}^2, \\
	& \le (1 - \frac{\mu (\gamma_{t-1})^2}{L}) \left[ f(X_{t - 1}) - f(X_*) \right], \\
	& \le (1 - \frac{\mu d_1^2}{L}) \left[ f(X_{t - 1}) - f(X_*) \right]
	\end{align*}
	Induct from $t = 1$ to $t = T$, we then have the Theorem.
\end{proof}


\section*{Theorem~\ref{the:rate:smtwc}}

First, we show that 
$\{ \FN{X_t - X_*} \}$
is upper-bounded. 

\begin{proposition} \label{pr:upbndx}
For $\{X_t\}$ generated by Algorithm~\ref{alg:smooth}, 
$\max_{t=1}^T \FN{X_t - X_*} \le d_2$
for some $d_2$.
\end{proposition}

\begin{proof}
	Let $A_t = f(X_t) - f(X_*)$,
	from Proposition~\ref{pr:descent}
	\begin{align}
	A_{t - 1} - A_t \ge \frac{d_1^2}{2 L} \FN{\nabla f(X_{t - 1})}^2.
	\label{eq:temp6}
	\end{align}
	
	Summing (\ref{eq:temp6}) from $t = 1$ to $T$, then
	\begin{align*}
	\sum_{t = 1}^T \frac{d_1^2}{2 L} \FN{\nabla f(X_{t-1})}^2 
	& \le \sum_{t = 1}^{T} \left( A_{t - 1} - A_t \right) \\
	& = A_0 - A_T \le A_0
	\end{align*}
	
	Since $f(X)$ is lower bounded, thus on $T = + \infty$, 
	we must have $\nabla f(X_{t - 1}) = 0$, 
	i.e. $X_t$ is a convergent sequence to $X_*$ and will not diverse.
	As a result, there must exist a constant $d_2$ such that $\FN{X_t - X_*} \le d_2$.
\end{proof}

Now, we start to prove Theorem~\ref{the:rate:smtwc}.
\begin{proof}
	From convexity, and since $f(X_*)$ is the minimum,
	therefore, we have
	\begin{align*}
	\left\langle X_{t - 1} - X_*, \nabla f(X_{t - 1}) \right\rangle 
	\ge f(X_{t - 1}) - f(X_*) \ge 0.
	\end{align*}
	
	Next, since $\nabla f(X_*) = 0$ and use Cauchy inequality $\left\langle X, Y \right\rangle \le \FN{X} \FN{Y}$, then
	\begin{align}
	& f(X_{ t - 1}) - f(X_*) 
	\notag \\
	& \quad \le \left\langle X_{t - 1} - X_*, \nabla f(X_{t - 1}) - \nabla f(X_*) \right\rangle
	\notag \\
	& \quad \le \FN{X_{t - 1} - X_*} \FN{\nabla f(X_{t - 1}) - \nabla f(X_*)}
	\label{eq:temp7}
	\end{align}
	
	Then, from Proposition~\ref{pr:upbndx},
	there exist a $d_2$ such that $\FN{X_{t - 1} - X_*} \le d_2$,
	combining it with (\ref{eq:temp7}), we have:
	\begin{align*}
	\FN{\nabla f(X_{t - 1}) - \nabla f(X_*)} \ge \frac{1}{d_2} \left[ f(X_{t - 1}) - f(X_*) \right] 
	\end{align*}
	
	Combine above inequality with Proposition~\ref{pr:descent},
	\begin{align*}
	A_t
	& \le A_{t - 1} - \frac{d_1^2}{2 L} \FN{\nabla f(X_{t - 1}) - \nabla f(X_*)}^2 \\
	& \le A_{t - 1} - \frac{d_1^2}{2 L d_2^2} A_{t-1}^2.
	\end{align*}
	
	Therefore, by Lemma~B.2 of \cite{shalev2010trading},
	the above sequence converges to $0$ of rate
	\begin{align*}
	f(X_T) - f(X_*) \le \frac{2 L d_2^2 A_0}{d_1^2 T A_0 + 2 L d_2^2}.
	\end{align*}
	which proves the Theorem.
\end{proof}


\section*{Theorem~\ref{the:rate:matcomp}}

%
%
%
%
\begin{proof}
	When restrict to $\Omega$ in \eqref{eq:mc}:
	\begin{align*}
		\ell(X) \le \; \ell(Y) + \left\langle \SO{X - Y}, \nabla \ell(Y) \right\rangle 
		+ \frac{1}{2} \FN{\SO{X - Y}}^2.
	\end{align*}
	
	On optimal $\nabla \ell(X_*) = 0$, using above inequality, then
	\begin{align*}
		\ell(X_{t-1}) \le \ell(X_*) + \frac{1}{2} \FN{\nabla \ell(X_{t-1}) - \nabla \ell(X_*)}^2
	\end{align*}
	
	As a result, we get 
	\begin{align}
		\FN{\nabla \ell(X_{t-1}) - \nabla \ell(X_*)}^2 \ge 2 \left[ \ell(X_{t-1}) - \ell(X_*)\right] 
		\label{eq:temp13}
	\end{align}
	
	Since $\ell$ is $1$-Lipschitz smooth and $\nabla \ell(X_*) = 0$, 
	together with Proposition~\ref{pr:descent} and (\ref{eq:temp13}), then
	\begin{align*}
		\ell(X_t) - \ell(X_*) 
		& \le \left[ \ell(X_{t-1}) - \ell(X_*) \right] - \frac{d_1^2}{2} \FN{\nabla \ell(X_{t - 1})}^2, \\
		& \le \left(1 - d_1^2\right) \left[ \ell(X_{t-1}) - \ell(X_*) \right].
	\end{align*}
	
	Induct from $t = 1$ to $t = T$, we then have
	\begin{align*}
		\ell(X_T) - \ell(X_*) 
		\le \left(1 - d_1^2 \right)^T \left[ \ell(X_0) - \ell(X_*) \right].
	\end{align*}
	
	Thus, we get the Theorem,
	and linear rate exists.
\end{proof}

%
%
%


\section*{Theorem~\ref{the:rate:gs}}
\label{app:the:rate:gs}

\noindent
Proof follows Theorem 5 at \cite{grubb2011generalized}.

\vspace{-20px}

\begin{proof}
\begin{eqnarray*}
\lefteqn{\FN{X_t - X_*}^2} \\
& = & \FN{X_{t - 1} - X_* - \eta_t h_t }^2 \\
& = & \FN{X_{t - 1} - X_*}^2 + \eta_t^2 \FN{h_t}^2 - 2 \eta_t \left\langle h_t, X_{t - 1} - X_* \right\rangle \\
& = & \FN{X_{t - 1} - X_*}^2 - 2 \eta_t \left\langle g_t, X_{t - 1} - X_* \right\rangle  \\
& & - 2 \eta_t \left\langle h_t - g_t, X_{t - 1} - X_* \right\rangle + \eta_t^2 \FN{h_t}^2.
\end{eqnarray*}

Rearranging items, we have:
\begin{eqnarray}
\lefteqn{\left\langle g_t, X_{t - 1} - X_* \right\rangle} \nonumber \\
& = & \frac{1}{2 \eta_t }\FN{X_{t - 1} - X_*}^2 - \frac{1}{2 \eta_t } \FN{X_t - X_*}^2
\nonumber \\
&& + \frac{\eta_t}{2} \FN{h_t}^2 - \left\langle 
h_t - g_t, 
X_{t - 1} - X_*
\right\rangle.
\label{eq:temp9}
\end{eqnarray}

As $f$ is $\mu$-strongly convex, 
\begin{eqnarray}
f(X_*) & \ge & f(X_{t - 1}) + \left\langle 
g_t, X_* - X_{t - 1}
\right\rangle 
\nonumber \\
&& + \frac{\mu}{2} \FN{X_{t - 1} - X_*}^2.  \label{eq:temp8} 
\end{eqnarray}

Sum \eqref{eq:temp8} from $t = 1$ to $T$, and using (\ref{eq:temp9})	
\begin{eqnarray}
\lefteqn{\sum_{t = 1}^T f(X_*) }\nonumber \\
& \ge & \sum_{t = 1}^T f(X_{t - 1})
+ \left\langle 
g_t, 
X_* - X_{t - 1}
\right\rangle
+ \frac{\mu}{2} \FN{X_{t - 1} - X_*}^2
\nonumber \\
& \ge & \sum_{t = 1}^T \left\lbrace f(X_{t - 1}) 
	- \frac{\eta_t}{2} \FN{h_t}^2
	- \left\langle X_* - X_{t - 1}, h_t - g_t \right\rangle \right\rbrace 
	\nonumber \\
& & + \frac{1}{2} \sum_{t = 1}^{T - 1} (\frac{1}{\eta_t} - \frac{1}{\eta_{t + 1}} + \mu) \FN{X_t - X_*}^2
	\nonumber \\
& & + \frac{1}{2} (\mu - \frac{1}{\eta_1}) \FN{X_0 - X_*}^2 
	+ \frac{1}{2 \eta_T} \FN{X_T - X_*}
\nonumber \\
& \ge & \sum_{t = 1}^T \left\lbrace f(X_{t - 1}) 
- \frac{\eta_t}{2} \FN{h_t}^2
- \left\langle X_* - X_{t - 1}, h_t - g_t \right\rangle \right\rbrace 
\nonumber \\
& &+ \frac{1}{2} \sum_{t = 0}^{T - 1} (\frac{1}{\eta_t} - \frac{1}{\eta_{t + 1}} + \mu) \FN{X_t - X_*}^2 
\nonumber\\
& &- \frac{1}{2\eta_0} \FN{X_0 - X_*}^2
\label{eq:temp10}
\end{eqnarray}
	
Recall, the step size is $\eta_t = c_1/t$,
then (\ref{eq:temp10}) becomes
\begin{eqnarray}
\lefteqn{\sum_{t = 1}^T f(X_*)}
\notag \\
&\ge &  \sum_{t = 1}^T f(X_{t - 1}) -
\frac{c_1}{2}\sum_{t = 1}^T \frac{1}{t} \FN{h_t}^2 
\notag \\
&  &+ \frac{1}{2}\left(\mu - \frac{1}{c_1}\right)
\sum_{t = 1}^{T} \FN{X_{t - 1} - X_*}^2
- \frac{1}{2c_1} \FN{X_0 - X_*}^2 
\notag \\
& & - \sum_{t = 1}^{T} \left\langle X_* - X_{t - 1}, h_t - g_t \right\rangle
\label{eq:temp11} 
\end{eqnarray}
where $\eta_0$ is can be picked up as $c_1$. 
For the second term in (\ref{eq:temp11}), it is simply bounded as
\begin{align}
\frac{c_1}{2}\sum_{t = 1}^T \frac{\FN{h_t}^2}{t} 
& \le \frac{c_1}{2}\sum_{t = 1}^T \frac{b_1^2}{t} 
\notag \\
& \le \frac{c_1 b_1^2}{2} \left( 1 + \log(T) \right).
\label{eq:temp15}
\end{align}

Let $\hat{c} = \frac{1}{2}\left(\mu - \frac{1}{c_1}\right)$,
since $c_1 \ge \frac{1}{\mu}$, thus $\hat{c} \ge 0$.
For last term in (\ref{eq:temp11}), 
we use $\FN{A}^2 + 2 \left\langle A, B \right\rangle  \ge - \FN{B}^2$,
let $A = X_t - X_*$ and $B = \frac{1}{2 \hat{c}}(h_t - g_t)$,
then
\begin{align}
& \hat{c} \sum_{t = 1}^T \left(  \FN{X_t - X_*}^2 
+ \frac{1}{\hat{c}} \left\langle X_t - X_*, h_t - g_t \right\rangle \right)
\notag \\
& \ge - \hat{c} \sum_{t = 1}^T \FN{\frac{1}{2 \hat{c}}(h_t - g_t)}^2
= -\frac{1}{4 \hat{c}} \sum_{t = 1}^T \FN{h_t - g_t}^2
\label{eq:temp14}
\end{align}

By definition of $h_t$ and assumption $\max_{t=1}^T \FN{g_t} \le b_1$,
for the first iteration ($t = 1$)
\begin{align*}
\FN{h_1 - g_1}^2
& = \FN{s_1 (u_1 v_1^{\top}) + g_1}^2 \\
& = (s_1)^2 + 2 s_1 \left\langle u_1 v_1^{\top}, g_1 \right\rangle + \FN{g_1}^2 \\
& \le 4 \FN{g_1}^2 = 4 b_1^2.
\end{align*}

In Algorithm~\ref{alg:nonsmooth}, we ensure
\begin{align*}
\FN{h_t - g_t}^2 \le \nu \FN{h_{t - 1} - g_{t - 1}}^2
\end{align*}
thus (\ref{eq:temp14}) becomes:
\begin{align}
\sum_{t = 1}^T \FN{h_t - g_t}^2 
& \le 4 \sum_{t = 1}^T \nu^{t - 1} b_1^2
\notag \\
& \le 4 \sum_{t = 1}^{+ \infty} \nu^{t - 1} b_1^2
= \frac{4 b_1^2}{1 - \nu}.
\label{eq:temp16}
\end{align}

Now, using (\ref{eq:temp15}) and (\ref{eq:temp16}), we can bound (\ref{eq:temp11}) as
\begin{align*}
\sum_{t = 1}^T f(X_*) 
& \ge \sum_{t = 1}^T f(X_{t - 1}) - \frac{c_1 b_1^2}{2} \left( 1 + \log(T) \right) \\
& - \frac{1}{2c_1} \FN{X_0 - X_*}^2
- \frac{b_1^2}{(1 - \nu)\hat{c}}.
\end{align*}

Thus, we can get convergence rate as
\begin{eqnarray*}
\lefteqn{\min_{t = 0, \cdots, T - 1} \left[ f(X_t) -  f(X_*) \right] T}
\\
& & \le \sum_{t = 1}^T \left[ f(X_{t - 1}) -  f(X_*) \right] 
\\
& & \le \left( 1 + \log(T) \right) \frac{c_1 b_1^2}{2} + \frac{1}{2 c_1} \FN{X_0 - X_*}^2 
+ \frac{b_1^2}{(1 - \nu)\hat{c}}.
\end{eqnarray*}

Finally, note that
\begin{eqnarray}
\min_{t = 0, \cdots, T} f(X_t) -  f(X_*)
\le \min_{t = 0, \cdots, T - 1} f(X_t) -  f(X_*)
\label{eq:temp4}
\end{eqnarray}

Let $b_2 = \frac{1}{2 c_1} \FN{X_0 - X_*}^2 + \frac{b_1^2}{(1 - \nu)\hat{c}} $, we get the theorem.
\end{proof}


\section*{Theorem~\ref{the:rate:gc}}
\label{app:the:rate:gc}

\noindent
Proof follows Theorem 6 at \cite{grubb2011generalized}.
First, we introduce below Proposition

\begin{proposition}[\cite{beardon1996sums}]
	\label{pr:sumpower}
	Given $k > 0$, the approximation of sum over natural numbers is:
	\begin{align*}
	\sum_{t = 1}^T t^k \le \frac{(T + 0.5)^{k + 1}}{k + 1}.
	\end{align*}
\end{proposition}

Now, we start to prove Theorem~\ref{the:rate:gc}.

\begin{proof}
	For weak convex convexity (obtained from \eqref{eq:temp10} with $\mu = 0$), we have
	\begin{align}
	\sum_{t = 1}^T f(X_*) 
	& \ge \sum_{t = 1}^T f(X_{t - 1}) - \frac{1}{2\eta_0} \FN{X_0 - X_*}^2
	\label{eq:temp17} \\
	& \quad  + \frac{1}{2} \sum_{t = 0}^{T - 1} \left(\frac{1}{\eta_t} - \frac{1}{\eta_{t+1}}\right) \FN{X_t - X_*}^2
	\notag \\
	& \quad -\sum_{t = 1}^T \frac{\eta_t}{2} \FN{h_t}^2 
	- \sum_{t = 1}^T \left\langle X_* - X_{t - 1}, h_t - g_t \right\rangle 
	\notag
	\end{align}
	where $\eta_0$ is picked up at $c_2$.
	The step size is $\eta_t = c_2/\sqrt{t}$, use it back into (\ref{eq:temp17}), we get
	\begin{eqnarray}
	\lefteqn{\sum_{t = 1}^T f(X_*)}
	\notag \\
	& \ge & \sum_{t = 1}^T f(X_{t - 1}) - \frac{1}{2 c_2} \FN{X_0 - X_*}^2
	\notag \\
	& & + \frac{c_2}{2} \sum_{t = 1}^{T - 1} \left(\frac{1}{\sqrt{t}} 
	- \frac{1}{\sqrt{t + 1}}\right) \FN{X_t - X_*}^2
	\notag \\
	& &- \sum_{t = 1}^T \frac{c_2}{2\sqrt{t}} \FN{h_t}^2 - \sum_{t = 1}^T \left\langle X_* - X_{t - 1}, h_t - g_t \right\rangle
	\notag \\
	& \ge & \sum_{t = 1}^T f(X_{t - 1}) - 
	\frac{c_2}{2} \sum_{t = 1}^{T}\frac{1}{\sqrt{t}} \FN{X_t - X_*}^2
	\notag \\
	& & - \sum_{t = 1}^T \frac{c_2}{2\sqrt{t}} 
	\FN{h_t}^2 - \sum_{t = 1}^T \left\langle X_* - X_{t - 1}, h_t - g_t \right\rangle
	\notag \\
	& \ge & \sum_{t = 1}^T f(X_{t - 1}) 
	- \frac{c_2}{2} \sum_{t = 1}^{T}\frac{1}{\sqrt{t}} \FN{X_t - X_*}^2
	\label{eq:temp18} \\
	& & - \frac{c_2}{2} \sum_{t = 1}^T \frac{1}{\sqrt{t}} \FN{h_t}^2 - \sum_{t = 1}^T \FN{X_* - X_{t - 1}} \FN{h_t - g_t}
	\notag
	\end{eqnarray}
	where the second inequality comes from
	the fact $1/\sqrt{t} - 1/\sqrt{t + 1} \ge - 1/\sqrt{t}$.
	For the last term in (\ref{eq:temp18}), since $\FN{X_* - X_{t - 1}} \le b_3$
	and $\FN{h_1 - g_1} \le 2 b_1$,	
	it can be bounded as
	\begin{align}
	\sum_{t = 1}^T & \FN{X_* - X_{t - 1}} \FN{h_t - g_t}
	\notag \\ 
	& \le b_3 \sum_{t = 1}^T \FN{h_t - g_t} 
	\le b_3 \sum_{t = 1}^T (\sqrt{\nu})^{t - 1} \FN{h_1 - g_1}
	\notag \\
	& \le 2 b_1 b_3 \sum_{t = 1}^{+ \infty} (\sqrt{\nu})^{t - 1}
	= \frac{2 b_1 b_3}{1 - \sqrt{\nu}}
	\label{eq:temp19}
	\end{align}
	
	Then, note that $\FN{h_t} \le \FN{g_t} \le b_1$, 
	use (\ref{eq:temp18}) and (\ref{eq:temp19})
	\begin{align*}
	\sum_{t = 1}^T f(X_*) 
	& \ge \sum_{t = 1}^T f(X_{t - 1}) - \frac{c_2(b_3^2 + b_1^2)}{2} \sum_{t = 1}^{T}\frac{1}{\sqrt{t}} - \frac{2 b_1 b_3}{1 - \sqrt{\nu}}
	\end{align*}
	
	Rearrange items in above inequality, from Proposition~\ref{pr:sumpower}:
	\begin{align*}
	\min_{t = 0, \cdots, T - 1} & \left[ f(X_t) - f(X_*) \right] \\
	& \le \frac{1}{T} \sum_{t = 1}^T \left[ f(X_{t-1}) - f(X_*) \right] \\
	& \le \frac{c_2(b_1^2 + b_3^2)}{2 T} \sum_{t = 1}^{T}\frac{1}{\sqrt{t}}
	 - \frac{2 b_1 b_3}{(1 - \sqrt{\nu}) T} \\
	& \le \frac{c_2(b_1^2 + b_3^2)}{2 \sqrt{T}} - \frac{2 b_1 b_3}{(1 - \sqrt{\nu}) T} 
	\end{align*}
	
	Finally, using \eqref{eq:temp4} we get the Theorem.
\end{proof}

\end{document}